\newif\ifcomments
\commentsfalse

\documentclass{article}
\usepackage[nonatbib,final]{rlc2025_prl_neurips}

\usepackage{amssymb}            
\usepackage{mathtools}          
\usepackage{mathrsfs}           
\usepackage{graphicx}           
\usepackage{subcaption}         
\usepackage[space]{grffile}     
\usepackage{url}                
\usepackage{lipsum}             

\usepackage{amssymb,amsthm,amsmath,bm,mathrsfs}
\usepackage{adjustbox}
\usepackage{enumitem}
\usepackage{etoc}
\usepackage{caption}
\usepackage{subcaption}
\usepackage{centernot}
\usepackage{complexity}
\usepackage{tabularx}   
\usepackage{multirow}   
\usepackage{array}      
\usepackage{booktabs}   
\usepackage{colortbl}   
\usepackage{xargs}      
\usepackage{tikz}
\usepackage{tikz-cd}
\usepackage{xurl}  
\usepackage{hyperref}  
\definecolor{dark-blue}{rgb}{0,0,0.7}  
\hypersetup{  
    colorlinks, linkcolor={dark-blue},
    citecolor={dark-blue}, urlcolor={dark-blue}
}
\usepackage{makecell}   
\usepackage[ruled,linesnumbered,noend]{algorithm2e} 
\usepackage{lipsum}
\usepackage{listings}
\usepackage{soul}       
\usepackage{minitoc}    
\usepackage[textsize=scriptsize]{todonotes}
\usepackage{pifont}     
\usepackage{wrapfig}
\usepackage{threeparttable}

\usepackage[nameinlink]{cleveref}

\usetikzlibrary{calc,positioning}

\newcommand{\coloursm}{blue}
\newcommand{\colourjz}{orange}
\newcommand{\colourtc}{cyan}
\newcommand{\colourdc}{purple}

\newcommand{\todocustom}[3]{\todo[linecolor=#2,backgroundcolor=#2!25,bordercolor=#2,#3]{#1}}

\newcommand{\removehide}[1]{}

\commentstrue
\ifcomments

\newcommand{\generalComment}[3]{\todocustom{{\bf #2}: #3}{#1}{inline,size=\small,caption={}}}  

\newcommandx{\nbsm}[2][1=]{\todocustom{#2}{\coloursm}{#1}}
\newcommandx{\nbdc}[2][1=]{\todocustom{#2}{\colourdc}{#1}}
\newcommandx{\nbjz}[2][1=]{\todocustom{#2}{\colourjz}{#1}}
\newcommandx{\nbtc}[2][1=]{\todocustom{#2}{\colourtc}{#1}}

\newcommand{\commentsm}[1]{\textcolor{orange}{({\bf SM:} #1)}}
\newcommand{\commentsmbox}[1]{\generalComment{\coloursm}{SM}{#1}}
\newcommand{\commentjz}[1]{\generalComment{\colourjz}{JZ}{#1}}
\newcommand{\commenttc}[1]{\generalComment{\colourtc}{TC}{#1}}
\newcommand{\commentdc}[1]{\generalComment{\colourdc}{DC}{#1}}

\else

\newcommand{\commentsm}[1]{}
\newcommand{\commentsmbox}[1]{}
\newcommand{\commentjz}[1]{}
\newcommand{\commenttc}[1]{}
\newcommand{\commentdc}[1]{}

\newcommandx{\nbsm}[2][1=]{}
\newcommandx{\nbdc}[2][1=]{}
\newcommandx{\nbjz}[2][1=]{}
\newcommandx{\nbtc}[2][1=]{}

\renewcommand{\todo}[1]{}

\fi

\renewcommandx{\nbsm}[2][1=]{}
\renewcommandx{\nbdc}[2][1=]{}
\renewcommandx{\nbjz}[2][1=]{}
\renewcommandx{\nbtc}[2][1=]{}

\input{support/macros.tex}

\newcommand{\citet}[1]{\cite{#1}}

\newcommand{\authorSep}{\quad}

\title{
  Language Models For Generalised PDDL Planning:\\Synthesising Sound and Programmatic Policies
}
\author{
  Dillon Z. Chen\textsuperscript{1,2,3} \authorSep
  Johannes Zenn\textsuperscript{1} \authorSep
  Tristan Cinquin\textsuperscript{1}\authorSep
  Sheila A. McIlraith\textsuperscript{1,3}
  \\
  $^{1}${Vector Institute} \qquad
  $^{2}${LAAS-CNRS, University of Toulouse} \qquad
  $^{3}${University of Toronto}
}

\newcommand{\rlapTableSize}{
  \small
}
\newcommand{\algFontSize}{
  \small
}

\newcommand{\covTableSize}{
  \footnotesize
}

\newcommand{\ssection}[1]{\textbf{#1}}

\begin{document}

\maketitle  

\begin{abstract}
  We study the usage of language models (LMs) for planning over world models specified in the Planning Domain Definition Language (PDDL). We prompt LMs to generate Python programs that serve as generalised policies for solving PDDL problems from a given domain. Notably, our approach synthesises policies that are provably sound relative to the PDDL domain without reliance on external verifiers. We conduct experiments on competition benchmarks which show that our policies can solve more PDDL problems than PDDL planners and recent LM approaches within a fixed time and memory constraint. Our approach manifests in the \lmplan{} planner which can solve planning problems with several hundreds of relevant objects. Surprisingly, we observe that LMs used in our framework sometimes plan more effectively over PDDL problems written in meaningless symbols in place of natural language; e.g. rewriting \pddl{(at dog kitchen)} as \pddl{(p2 o1 o3)}. This finding challenges hypotheses that LMs reason over word semantics and memorise solutions from its training corpus, and is worth further exploration.
\end{abstract}



\section{Introduction}
\label{sec:introduction}
%

\emph{AI automated planning} (AP)~\cite{ghallab.etal.2004,geffner.bonet.2013} refers to the class of sequential decision-making problems formally specified by symbolic models in specification languages such as the Planning Domain Definition Language (PDDL)~\cite{mcdermott.etal.1998,haslum.etal.2019}, and typically solved using heuristic search techniques.
Although AP is intractable~\cite{chapman.1987,bylander.1994,erol.etal.1995}, agents are often tasked with tractable families of planning problems---problems that share a common set of actions, transition system, and typed objects, and whose structural properties can be exploited to construct plans or policies that solve these families of problems.
For example, interplanetary rovers must continuously explore and collect information about planets with minimal communication and time for activity planning~\cite{bresina.etal.2005}, while logistics companies such as UPS$^{\text{TM}}$ ship and deliver packages and scale to over 20 million packages every day across 200 countries and territories in 2024~\cite{ups.2025}. 
In such applications, there is much need for \emph{planning efficiently at scale} to solve time-intensive problems, and \emph{automated plan synthesis}\nbsm{the mention of program synthesis seems to come out of the blue. Do you want to say "plan synthesis" instead?}\nbdc{Originally yes to fit the theme of the workshop. But now that you point it out, it reads better with program $\ra$ plan since we introduce programs in the next sentence.} to minimise dependency on human intervention.
Generalised planning (GP) exactly encapsulates this problem of automatically generating plans as programs that address families of related planning problems~\cite{levesque.2005,srivastava.etal.2008,srivastava.etal.2011a,bonet.etal.2009,hu.degiacomo.2011,bonet.etal.2019,illanes.mcilraith.2019,celorrio.etal.2019,cui.etal.2021,frances.etal.2021,drexler.etal.2022}.

Recently, immense progress in the development of language models (LMs)~\cite{brown.etal.2020,chowdhery.etal.2023} has revealed strong emergent capabilities in reasoning and problem solving~\cite{kojima.etal.2022,wei.etal.2022}.
This progress has given rise to significant advances across various areas of artificial intelligence and problem solving, including algorithmic discovery~\cite{romeraparedes.etal.2024,alexandernovikov.balog.2025}, code generation~\cite{chen.etal.2021,nijkamp.etal.2023,liu.etal.2024}, competitive programming~\cite{openai.2025}, and embodied intelligence~\cite{driess.etal.2023,intelligence.etal.2025}.
However, several studies have suggested that current LMs are incapable of solving long-horizon planning problems~\cite{valmeekam.etal.2023,valmeekam.etal.2024}.

In this work, we leverage LMs to synthesise Python programs as solutions to GP problems. 
Similarly to prior work in LMs for GP~\cite{silver.etal.2024}, we encode GP problems in PDDL, the de facto specification for symbolic AP problems.
We assume that the PDDL models are given but they can be synthesised with minimal human intervention from unstructured data such as natural language~\cite{collins.etal.2022,lin.etal.2023,guan.etal.2023,oswald.etal.2024,huang.etal.2025,tantakoun.etal.2025}, images~\cite{xi.etal.2024,athalye.etal.2025}, and environment interactions~\cite{verma.etal.2021,verma.etal.2022,silver.etal.2023,sreedharan.katz.2023,liang.etal.2025}.

We study LM-generated programs as (1) value functions following~\citet{correa.etal.2025}, and (2) policies for solving GP problems represented in PDDL.
We use LMs to generate Python programs implementing value functions to guide heuristic search, and policies as reactive controllers that specify an action to take from a set of applicable actions in a given state.
Notably, we guarantee that the synthesised policies are sound (any returned solution is correct in relation to the PDDL domain theory) by restricting them to predict actions that are only applicable at the input state, and furthermore to improve search performance.
Our approach manifests in an LM planner that outperforms state-of-the-art planners on total number of problems solved within a given time and memory limit from the recent International Planning Competition Learning Track~\cite{taitler.etal.2024}.
We further study the effect of symbolic and semantic representations of planning problems for LMs used in our framework.
We observe that surprisingly and in contrast to observations in previous works~\cite{valmeekam.etal.2023a,silver.etal.2024}, LMs used in our framework can sometimes match or even perform better on planning problems encoded using meaningless symbols (e.g., \pddl{o1} for \pddl{dog} or \pddl{p2} for \pddl{at}).
This is a provocative finding worth further exploration since it could suggest that the LM has learned to do some form of symbolic planning or reasoning.
Our contributions are summarised as follows.

\begin{figure}[t]
  \centering
  \includegraphics[width=.9\textwidth]{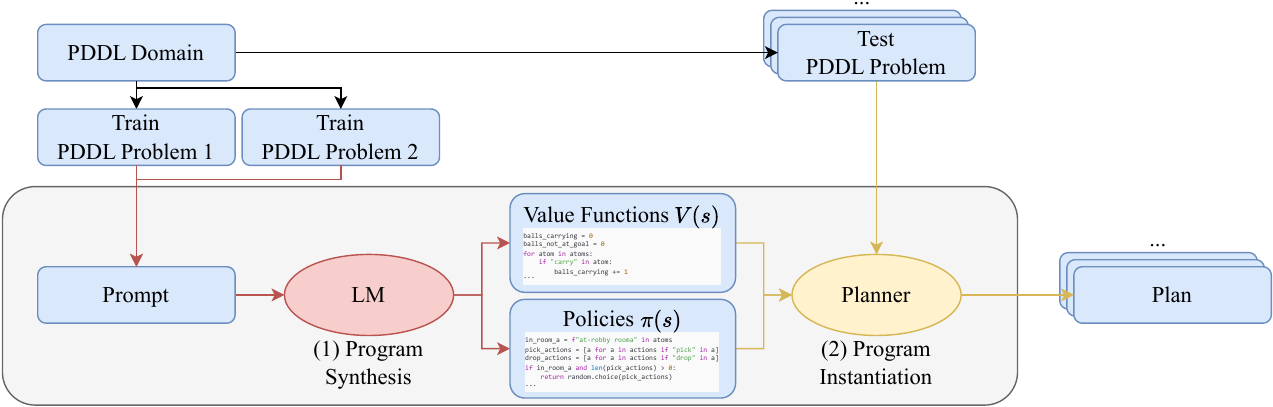}
  \caption{
    Pipeline for planning with LM-generated value functions and policies.
    The architecture is enclosed in gray.
    (1) A domain, 2 example PDDL problems, and a prompt is input into the LM which outputs a Python program representing a value function or policy.
    (2) The program is then used to help plan for PDDL problems from the same domain.
    See \Cref{sec:method} for more details.
  }
  \label{fig:pipeline}
\end{figure}

\begin{itemize}[leftmargin=*]
  \item We use LMs to generate code implementing \emph{sound} policies and compare its performance to generating value functions used in heuristic search for planning, with results generally favouring the former approach.
    We further use LM-generated policies to improve heuristic search performance.
  \item We conduct experiments demonstrating that LMs used in our framework can sometimes plan better on PDDL problems written with only symbols in place of natural language.
    This observation challenges hypotheses from previous works that LMs can only plan with languages with semantic meaning and memorise solutions from their training corpus.
\end{itemize}


\section{Preliminaries: Classical Planning, PDDL, and Generalised Planning}
\label{sec:background}

In this section we introduce the necessary background on AP, followed by the problem we tackle in this paper: generalised planning.
AP refers to the class of sequential decision-making problems using models represented in formal, symbolic languages and the methodologies for solving them.
We begin by introducing the abstract classical planning problem model, a special case of AP problems where the objective is to find a sequence of actions that transitions an agent from an initial state to a goal state under a fully-observable, deterministic transition system.
Although probabilistic extensions of planning exist~\cite{bertsekas.tsitsiklis.1991,sanner.2010,mausam.kolobov.2012}, we focus our attention on classical planning for ease of presentation.
Following this, we provide an informal introduction to PDDL for encoding planning problems.
We then introduce heuristic search as a state-of-the-art method for solving planning problems and how they can be viewed analogously to value functions in reinforcement learning (RL).
For those familiar with RL, \Cref{tab:analogies} draws analogies between common RL and AP methodologies.
We conclude this section by formalising the GP problem which is concerned with synthesising programmatic plans for solving families of related planning problems.

\begin{table}[t]
  \centering
  \caption{Analogies between RL and AP methodologies.}
  \label{tab:analogies}
  \rlapTableSize
  \begin{tabularx}{\textwidth}{l l l}
    \toprule
    Methodology & Reinforcement Learning (RL) & AI Planning (AP) \\
    \midrule
    Action space & Policy approximation & Policy synthesis \ref{ssec:policies} \\
    State space & Value function approximation & Heuristic search \ref{ssec:heuristics} \\
    State \& action & Actor-critic algorithms & Heuristic search w. preferred operators \ref{ssec:both} \\
    \bottomrule
  \end{tabularx}
\end{table}

\ssection{Planning problem}
A classical planning problem is a deterministic state transition model concerned with driving a given initial state into a goal state via a sequence of actions.
Following the notation in \cite{geffner.bonet.2013}, a \defn{(classical) planning problem} is a tuple $\problem = \gen{\states, \actions, \transition, \goals, \init, \cost}$ where
$\states$ is a set of states,
$\actions$ is a set of actions,
$\transition : \states \times \actions \ra \states \cup \set{\bot}$ is a \defn{deterministic transition function} where $\transition(s,a) = \bot$ represents that the action $a$ is not applicable in the state $s$,
$\goals \subseteq \states$ is a non-empty set of goal states,
$\init \in \states$ is the initial state, and
$\cost: \states \times \actions \ra \R_{\geq 0}$ is a cost function.
The set of \emph{applicable actions} of a state $s$ is defined by $A(s) = \set{a \mid \transition(s, a) \not= \bot} \subseteq A$ and the \emph{successors} of $s$ is defined by $\set{\transition(s, a) \mid a \in A(s)} \subseteq S$.
A solution or \emph{plan} $\plan$ for a planning problem is a finite sequence of actions $a_0, \ldots, a_n$ such that $\transition(s_i, a_i) = s_{i+1} \not= \bot$ for $i = 0, \ldots, n$ where $s_0 = \init$ and $s_{n+1} \in \goals$.
In this paper, we focus on \emph{satisficing planning} which refers to the problem of finding satisficing solutions for $\problem$, i.e. any plan for the problem.

\ssection{Planning representations: PDDL}\label{sssec:pddl}
Planning problems are often represented in a first-order formal language, the most common being PDDL.
Details of the PDDL syntax are not necessary for the understanding of the paper but one should note that fragments of PDDL planning range from being EXPSPACE-hard~\cite{erol.etal.1995} to undecidable~\cite{helmert.2002}.
This is because PDDL provides compact encodings of transition models that can be exponential in the size of the input files or greater.
A planning problem is represented by a PDDL \emph{domain}, providing a compact encoding of the transition function in terms of a set of object types, predicates, numeric functions, and actions, and a PDDL \emph{problem}, specifying a finite set of objects, an initial state, and a goal condition.
For example, a package delivery {\bf domain} can contain
\begin{itemize}[noitemsep,leftmargin=*]
  \item the types \pddl{object}, \pddl{vehicle}, \pddl{location}, \pddl{package},
  \item the predicates \pddl{(at ?o - object ?l - location)}, \pddl{(in ?p - package ?l - location)},
  \item the numeric functions \pddl{(capacity ?v - vehicle)}, \pddl{(weight ?p - package)}, and
  \item the following action schema, among others, for loading a package into a vehicle
\end{itemize}

\begin{lstlisting}[
    basicstyle=\ttfamily\scriptsize,
    breaklines=true,
]
(:action pick-up
 :parameters (?v - vehicle ?l - location ?p - package)
 :precondition (and (at ?v ?l) (at ?p ?l) (>= (capacity ?v) (weight ?p)))
 :effect (and (not (at ?p ?l)) (in ?p ?v) (decrease (capacity ?v) (weight ?p))))
\end{lstlisting}

Next, a package delivery {\bf problem} can consist of a truck \pddl{truck2 - vehicle} that is at the depot \pddl{(at truck2 depot)}, has a specified capacity \pddl{(= (capacity truck2) 5)}, and is carrying some packages \pddl{(in truck2 package3) \& (in truck2 package1)}.
A PDDL goal condition consists of a conjunction of ground atoms and inequalities of expressions of numeric functions, such as \pddl{(at package1 office) \& (>= (capacity truck2) 7)}.
A goal condition induces a set of goal states as a state that satisfies a goal condition is considered a goal state.

\ssection{Heuristic search}
Heuristic search is the main driver of current state-of-the-art planners~\cite{richter.westphal.2010,seipp.etal.2020}, with roots dating back to early planning systems~\cite{fikes.nilsson.1971,laborie.ghallab.1995,mcdermott.1996,bonet.etal.1997}.
A heuristic function is a function $h: \states \to \R \cup \set{\infty}$ which estimates the cost to go from a state $s$ to a goal state in $\goals$, and returns the value $\infty$ estimating that there is no plan from $s$ to a goal.
A heuristic is safe if $h(s) = \infty$ only if there is no plan from $s$ to a goal.
The optimal heuristic $h^*$ returns the \textbf{minimal} plan cost from a state, if a plan exists, and $\infty$ otherwise.
A heuristic function is analogous to an approximate value function in RL which estimates the \textbf{maximal} reward from a state where one views the cost of a plan as a negative reward.
Thus, $h^*$ is the optimal value function and induces a policy $\pi: \states \to \actions$ that executes a successor state with the lowest value, breaking ties arbitrarily.
However, $h^*$ is intractable to compute so heuristics are instead used to guide search.
The \defn{Greedy Best First Search (GBFS)} algorithm searches for a path over the graph induced by the transition system of a planning problem from an initial state to a goal state, guided by a heuristic function.
It consists of a priority queue initialised with the initial state as the only element, and a main loop that performs the following steps while the queue is non-empty:
\begin{enumerate}[label=({\arabic*}),leftmargin=*,noitemsep]
  \item pop a state $s$ with the lowest heuristic value (breaking ties arbitrarily) from the queue,
  \item generate the successors of $s$ via all applicable actions, and
  \item check if a successor $s'$ is a goal, in which case terminate with the plan to $s'$, and otherwise add $s'$ to the queue if it has not been seen before.
\end{enumerate}
The algorithm determines a problem is unsolvable if the main loop terminates which only occurs if there are finitely many states.
GBFS is \emph{sound} (any returned solution is correct) and \emph{complete} (a solution is returned if it exists) for planning problems with finite state spaces.
A$^*$ search~\cite{hart.etal.1968} is an optimal search algorithm when used with an admissible heuristic.
We do not use it in our experiments because LM-generated heuristics are not guaranteed to be admissible.

\ssection{Problem statement: generalised planning}
Recall that planning problems are specified by a PDDL domain and problem.
Let $\domain$ denote a PDDL domain, and $\problem$ a problem associated with $\domain$, where we say that the problem $\problem$ \emph{belongs} to $\domain$.
A \emph{generalised planning (GP) problem} is a tuple $\gen{\trainProblems, \testProblems}$ where $\trainProblems$ is a finite set of training problems and $\testProblems$ a (possibly infinite) set of testing problems belonging to the same domain.
A solution to a GP problem is a policy (here a program) that is synthesised from $\trainProblems$ and can be instantiated on and solve each problem $\problem \in \testProblems$.
A key attribute of GP problems is that problems in $\testProblems$ are larger in terms of number of objects and more difficult than to solve than problems in $\trainProblems$.
Thus, GP is as an out-of-distribution learning task.

\nbsm{Remind me to tell you about the "eyeball rolling" of the term "out-of-distribution" in the programmatic RL workshop.  It's fine/good to use here.}


\section{LM-Generated Python Programs for Generalised Planning}
\label{sec:method}
In this section we describe our approach for GP which employs LMs to generate code as programs representing value functions and policies for use in planning.
Notably, all approaches to be described next are sound algorithms and are furthermore complete when used with complete search.

\ssection{LMs for GP}
Our approach consists of two modules for solving a GP problem $\gen{\trainProblems, \testProblems}$ as illustrated in \Cref{fig:pipeline}: \textbf{(1)} a {program synthesis} module (\Cref{ssec:synthesise}) and \textbf{(2)} a {program instantiation} module (\Cref{ssec:instantiate}).
The \emph{program synthesis} module (red in \Cref{fig:pipeline}) takes as input the training problems $\trainProblems$, corresponding domain $\domain$, and a natural language prompt, and outputs a program implementing a value function or policy.
The \emph{program instantiation} module (yellow in \Cref{fig:pipeline}), takes as input a problem $\problem \in \testProblems$ and a program from the previous module and outputs a plan $\plan$ for $\problem$.
Note that the LM is queried for a program once per domain $\domain$ associated with the GP problem while the planner module is called for every problem $\problem$ in $\testProblems$.

\subsection{LMs for program generation}\label{ssec:synthesise}
Both the value function and policy programs are prompted to be generated as Python classes which extend a class containing the domain $\domain$ associated with the GP problem to be solved.
We follow a setup similar to~\cite{correa.etal.2025} for generating code as programs extended to both value functions and policies.
More specifically, for any given domain, we prompt an LM for a program as a value function or policy with the following content:
\begin{enumerate}[label=(\roman*),leftmargin=*,noitemsep]
    \item instructions for generating code for a program as a value function or policy,
    \item the domain $\domain$ corresponding to the GP problem and two problems in $\trainProblems$ in PDDL,
    \item an example PDDL file for the Gripper domain~\cite{mcdermott.2000} and a Gripper problem,
    \item an example Python class encoding a value function or policy for Gripper.
\end{enumerate}

Differently to~\cite{correa.etal.2025} we only provide example files for a single domain instead of two domains (Gripper and Logistics).
Gripper is a simple PDDL domain consisting of two rooms and a set of balls located in one room. The objective is for a robot to move all balls in one room to the other, subject to capacity constraints.
Logistics is a more complex domain than Gripper which commands a fleet of planes and trucks for delivering packages across various cities and locations.
We emphasise that the example Gripper files and Python class (iii-iv) are used to provide in-context learning~\cite{dong.etal.2024} for the LM to understand the syntax of PDDL and the Python class structure, but does not provide any information about solving the GP problem associated with $\domain$.

\subsection{Sound planning with LM-generated programs}\label{ssec:instantiate}
Next, we describe how to
\textbf{(a)} ensure that LM-generated policies are sound when used as reactive controllers,
\textbf{(b)} use value functions for sound and complete planning and
\textbf{(c)} combine both value functions and policies together for sound and complete planning and boosting performance over \textbf{(b)}.

\newcommand{\algoWidth}{0.475\textwidth}
\newcommand{\topvskip}{
    \vskip-1ex
}
\newcommand{\botvskip}{
    \vskip-2ex
}
\newcommand{\popH}{\mathit{popH}}
\newcommand{\queue}{\mathit{q}}
\newcommand{\queueH}{\queue_H}
\newcommand{\queueP}{\queue_P}
\newcommand{\push}{\mathit{push}}
\renewcommand{\insert}{\mathit{insert}}
\newcommand{\visited}{\mathit{visited}}

\setlength{\algomargin}{0em}
\begin{wrapfigure}{L}{\algoWidth}
  \topvskip
  \begin{minipage}{\algoWidth}
    \begin{algorithm}[H]
      \algFontSize
      \DontPrintSemicolon
      \caption{Greedy Best First Search (GBFS) with a value function and policy \ref{ssec:both}}
      \label{alg:gbfs}
      \KwIn{
        Planning problem $P=\gen{\states, \actions, \transition, \goals, \init, \cost}$,
        value function $h$, and
        policy $\pi$.
      }
      \KwOut{
        A plan $\plan$ or failure if no plan exists.
      }
      \lIf{$\init \in \goals$}{
        \Return{$\emptyset$} \label{line:trivial}
      }
      $\queueH \la [\init]; \queueP \la []; \visited \la \set{\init}; \popH \la \top$ \label{line:initend} \\
      \While{$\queueH$ or $\queueP$ is not empty}{ \label{line:while}
        \lIf{$\popH = \top$}{$s \la \argmin_{s \in \queueH} h(s)$}
        \lElse{$s \la \argmin_{s \in \queueP} h(s)$}
        $\popH \la \;!\popH$ \label{line:alternateend} ;
        $\queueP.push(\pi(s))$ \label{line:policy} \\
        \For{$a \in A(s)$}{ \label{line:actions}
          $s' \la \transition(s, a)$ \label{line:successors} \\
          \lIf{$s' \in \visited$}{
            \textbf{continue} \label{line:seen}
          }
          \lIf{$s' \in \goals$}{
            \Return{\textit{extract plan to $s'$}}  \label{line:goal}
          }
          $\visited.\insert(s')$; $\queue.\push(s')$ \label{line:rem2} \\
        }
      }
      \Return{\textit{failure}} \label{line:failure}
    \end{algorithm}
  \end{minipage}
  \botvskip
\end{wrapfigure}

\ssection{\textlabel{(3.2.a)}{ssec:policies} Sound policies as reactive controllers}
Given a planning problem $P=\gen{\states, \actions, f, \goals, s_I, c}$, the policy program \plmp{} is instantiated on $P$ to represent a policy $\pi: \states \to \actions$ that takes as input a state $s$ and its applicable actions $A(s)$\nbsm{If writing again might consider using notation that is suggestive of a set. But no big deal.} and outputs an action $\pi(s) \in A(s)$.
We furthermore have wrapper code around \plmp{} such that if due to an error or mistake in the generated code and $\pi(s) \notin A(s)$, we choose a random action from $A(s)$ instead.
We use the policy program in the usual way for a policy via rollout.
Specifically, we repeatedly apply the operation $s = f(s, \pi(s))$ starting from the initial state $s = s_I$ of a planning problem until either a goal is reached (i.e. $s \in \goals$) or no applicable actions exist.
Notably, this approach is sound, meaning that any plan returned by this procedure is valid.

\begin{theorem}
    Approach \ref{ssec:policies} is sound with respect to an input planning problem $P$.\nbsm{Do you need the word "input"? It's not mentioned in 3.2.a.  Also, it might help to remind the reader what the "input planning problem $P$ is by listing the tuple. This will make it more clear that the transition function $f$ is part of the problem definition.}
\end{theorem}
\begin{proof}[Proof sketch]
  This is because all predicted actions are applicable at their current state.
  Thus any sequence of actions generated by the policy is applicable from the initial state.
  Furthermore, a plan is only returned when the goal is reached so any returned action sequence reaches the goal.
\end{proof}

\ssection{\textlabel{(3.2.b)}{ssec:heuristics} Sound and complete value functions in search}
Equivalently to~\cite{correa.etal.2025}, the value function program \hlmp{} consists of a method representing a heuristic function $h: \states \to \R \cup \set{\infty}$\nbsm{Should it be positive reals? -- not critical.}\nbdc{It does not matter for Greedy Best First Search; i.e. negatives are fine}\nbsm{Only if you know there are no negative cycles. Otherwise, if I remember correctly, you may loop in a cycle. (My memory is foggy. Please recheck.)}\nbdc{Again, it's fine for GBFS if you not reopen nodes.} that takes as input a state $s$ from $\problem$ and outputs a value $h(s)$ used in GBFS described in \Cref{sec:background}.
To ensure the output heuristic is safe, we have wrapper code that converts $\infty$ outputs to a large constant value.
Thus, we have the following property that this approach is sound and complete.

\begin{theorem}
  Approach \ref{ssec:heuristics} is sound with respect to an input planning problem $P$ and complete if the state space of $P$ is finite.
\end{theorem}
\begin{proof}[Proof sketch]
  This follows from the fact that GBFS is sound and complete for finite state spaces when used with a safe heuristic, and that the generated value function program used as a heuristic is ensured to be safe by disallowing $\infty$ values.
\end{proof}

\ssection{\textlabel{(3.2.c)}{ssec:both} Sound and complete value functions and policies in search}
We now describe how to combine LM-generated policies and value functions with search.
The main idea is that GBFS can be extended to two queues from which nodes are popped in a round robin fashion: one each for a value function \hlmp{} and policy \plmp{}.
Indeed, the usage of multiple queues have been explored in previous work with multiple heuristics~\cite{roeger.helmert.2010} or \emph{preferred operators}, actions that are deemed useful for achieving the goal within the computation of a heuristic function~\cite{hoffmann.nebel.2001}.
The algorithm is summarised in \Cref{alg:gbfs}, which begins by checking if the problem is trivially solvable (\Cref{{line:trivial}}) before initialising the two queues $\queueH$ and $\queueP$ representing a queue containing successors of any expanded states, and a state predicted by the policy $\pi$, respectively (\Cref{line:initend}).
The main loop alternates between popping states from the two queues, followed by pushing the state predicted by the policy into the $\queueP$ queue (\Crefrange{line:while}{line:alternateend}), and the remainder of the original GBFS algorithm for the $\queueH$ queue (\Crefrange{line:actions}{line:rem2}) described in \Cref{sec:background}.
This approach is also sound and complete.

\begin{theorem}
  Approach \ref{ssec:both} is sound with respect to an input planning problem $P$ and complete if the state space of $P$ is finite.
\end{theorem}
\begin{proof}[Proof sketch]
  Extending GBFS with multiple queues preserves the same soundness and completeness properties of GBFS as the search is still exhaustive for finite state spaces.
\end{proof}


\newcommand{\question}[1]{{\textbf{Q:} #1}}
\newcommand{\answer}[1]{\newline \textbf{A:} {\color{magenta}{#1}}}

\section{Experiments}
\label{sec:experiments}
We conduct experiments to address the following questions.
\begin{enumerate}[label=(\textbf{\arabic*}),leftmargin=*]
    \item \question{Which of LM-generated value functions used for search or policies used as is solve more planning problems within a fixed time limit, and how do they compare to PDDL planners?}
          \answer{Policies are faster for simpler problems, while search with value functions solve more complex problems. LM programs are competitive with PDDL planners on easy domains.}
    \item \question{How important is soundness and completeness for planning performance? }
          \answer{Soundness is important but completeness is not always necessary.}
    \item \question{Are LMs planning over word semantics or logical symbols? }
          \answer{LMs are shown to be capable to reason over PDDL planning problems represented in either semantic or symbolic text, but more experimentation is required to draw conclusive results.}
\end{enumerate}

\begin{wrapfigure}{R}{0.35\textwidth}
    
    \vskip-5pt

    \centering
    \includegraphics{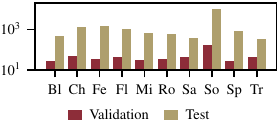}
    \vskip-3pt
    \caption{
        Number of objects (log $y$-axis) of validation and test problems across domains.
    }
    \label{fig:sizes}

    \vskip-5pt
    
\end{wrapfigure}

\ssection{Domains}
We evaluate the effectiveness of LM-generated value functions and policies on 10 standard PDDL planning domains
(\bl ocksworld, \ch ildsnack, \fe rry, \fl oortile, \mi conic, \ro vers, \sa tellite, \so koban, \sp anner, \tr ansport)
with validation and testing problems (of increasing difficulty in terms of number of objects) taken from the Learning Track of the 2023 International Planning Competition~\cite{taitler.etal.2024}.
\Cref{fig:sizes} summarises the ranges of problem sizes across the domains, noting that the testing sizes are up to two orders of magnitude larger than validation problems.

\ssection{Implementation}
For our experiments, we implement a planner from scratch, namely \lmplan{}, which prompts LMs for programs and both (a) uses value functions for heuristic search and (b) executes policies as reactive controllers.
\lmplan{} also implements the search using both value functions and policies introduced in \Cref{alg:gbfs}.
The implementation consists of a combination of Python for heuristic evaluation and C++ for data structure representations of planning components.
We use the SQLite library~\cite{hipp.2020} to compute actions applicable for a state in \Cref{line:actions}, as planning states can be viewed as databases and applicable action generation as database queries~\cite{correa.etal.2020}.

\ssection{Validation for selecting LM-generated programs}
With regards to LMs, we experiment with DeepSeek-R1~\cite{deepseekai.etal.2025}, Gemini 2.0 Flash, and Gemini 2.5 Flash Preview 04-17~\cite{geminiteamgoogle.etal.2023}.
Similarly to~\citet{correa.etal.2025}, we perform a validation procedure to select the best value function and policy for each domain out of several generated programs.
Each LM is called 10 times to generate a value function and policy program for each domain.
For each domain, the best value function (resp. policy) is then selected by taking the best average score of the program when used for search (resp. rollout) on 10 small, training problems.
We choose a scoring function that favours models that solve the training problems quickly, as the time to generate a solution is a major performance metric for satisficing planning.
Specifically, the \textlabelEmph{validation score}{ssec:valmetric} for each problem and program is given by $(1+\log(t+1))^{-1}$
if the program used for search or rollout solves a given problem in $t$ seconds for $t<60$ seconds, and $0$ otherwise.
\Cref{app:money} lists prompting costs, \Cref{app:best-programs} shows LMs that generated the selected program, and \Cref{sec:generation_time} reports the corresponding generation times.

\textlabel{Approaches}{ssec:approaches}
We evaluate our proposed approaches using LM-generated programs in the \lmplan{} planner ($\bullet$) and compare them to traditional PDDL planners ($\circ$) listed as follows:
\begin{itemize}[
        noitemsep,
        leftmargin=*
    ]
    \item[$\circ$] \hff{}: GBFS with the FF heuristic~\cite{hoffmann.nebel.2001},
    \item[$\circ$] \cps{}: reported results for LM-generated programs for heuristic search by~\citet{correa.etal.2025},
    \item[$\circ$] \wlgoose{}: GBFS with state-of-the-art value functions learned with the Weisfeiler-Leman graph kernel for generating features from planning tasks~\cite{chen.etal.2024a} and Gaussian process regression for computing linear model weights,
    \item[$\circ$] \lama{}: a state-of-the-art, general-purpose planner that uses multiple heuristics, queues and several optimisation techniques~\cite{richter.westphal.2010}.
    \item[$\bullet$] \heuristicMode{}: GBFS using an LM-generated program as a value function,
    \item[$\bullet$] \policyMode{}: rollout of an LM-generated program as a policy,
    \item[$\bullet$] \policyHeuristicMode{}: GBFS with two queues associated with \hlmp{} and \plmp{},
    \item[$\bullet$] \portfolio{}: a choice between \policyMode{} and \policyHeuristicMode{} depending on whether \plmp{} or \hlmp{} by themselves respectively achieves the higher average \ref{ssec:valmetric}.
\end{itemize}

There are 90 testing problems in each of the 10 domains for a total of 900 testing problems.
Each approach is run on Intel Xeon Platinum 8268 cores for 1800 seconds and with a memory budget of 8GB for each problem.
An exception is \cps{} where we report results from the original paper using AMD EPYC 7742 cores as their code is not yet publicly available.
We report the coverage metric in \Cref{tab:coverage}, the number of problems solved within the given computational budgets.

\begin{table}[t]
    \caption{
        Coverage ($\uparrow$) of existing planners (top) and \lmplan{} approaches introduced in this paper (bottom), see \ref{ssec:approaches} for details.
        The S/C column represents if an approach is sound/complete.
        Best values in each column are highlighted, where 90 is the best achievable score per domain.
    }
    \newcommand{\covTableWidth}{\textwidth}
    \tabcolsep 2pt
    \small
    \label{tab:coverage}

\begin{threeparttable}
  \begin{tabularx}{\covTableWidth}{p{2cm} *{2}{p{.5cm}} *{11}{Y}}
    \toprule
    & S & C
    & \bl & \ch & \fe & \fl & \mi & \ro & \sa & \so & \sp & \tr & $\Sigma$ \\
    \midrule
    \hff & \ding{51} & \ding{51} &
28 & 26 & 68 & \first{12} & \first{90} & 34 & 65 & 36 & 30 & 41 & 430 \\
    \cps & \ding{51} & \ding{51} &
66 & 22 & \zerocell{} & 4 & \first{90} & 32 & \zerocell{} & 30 & 70 & 59 & $^*$373 \\
    \wlgoose & \ding{51} & \ding{51} &
75 & 29 & 76 & 2 & \first{90} & 37 & 53 & 38 & 73 & 29 & 502 \\
    \lama & \ding{51} & \ding{51} &
61 & \first{35} & 68 & 11 & \first{90} & \first{67} & 89 & \first{40} & 30 & 66 & 557 \\
    \midrule
    \heuristicMode & \ding{51} & \ding{51} &
33 & 15 & 59 & 2 & 63 & 32 & 60 & 32 & 46 & 55 & 397 \\
    \policyMode & \ding{51} & \ding{55} &
\first{90} & 11 & \first{90} & 0 & \first{90} & 12 & \first{90} & 0 & \first{90} & \first{90} & 563 \\
    \policyHeuristicMode & \ding{51} & \ding{51} &
36 & 19 & 66 & 2 & 72 & 35 & 62 & 34 & 47 & 51 & 424 \\
    \portfolio & \ding{51} & \ding{55} &
\first{90} & 19 & \first{90} & 2 & \first{90} & 35 & \first{90} & 34 & \first{90} & \first{90} & \first{630} \\
    \bottomrule
  \end{tabularx}
  \begin{tablenotes}
  \item[$\ast$] Note that \cps{} does not support the \ch{} and \sa{} domains, indicated by \zerocell{}, such that $\sum$ is not representative of its best performance.
  \end{tablenotes}
\end{threeparttable}
\end{table}

\ssection{(1) Do value functions for search or policies generated by LMs solve planning problems faster?}
We note from \Cref{tab:coverage} that policies \plmp{} vastly outperform value functions \hlmp{} for more domains but otherwise both are complementary overall.
More specifically execution of the policy \plmp{} solves all problems for 6 domains (Blocksworld, Ferry, Miconic, Satellite, Spanner and Transport), but struggles to solve problems from the remaining domains.
Regardless, \plmp{} already outperforms a state-of-the-art planner, LAMA, in overall coverage (563 against 557) as well as in number of domains (6 against 4).
On the other hand, search with the value function \hlmp{} is complete and hence provides more well-rounded performance across domains by solving at least 30 problems from each domain except Childsnack and Floortile, but has a lower overall cumulative coverage.
A similar statement can be made by the baseline planners which all employ heuristic search with value functions.
However, the validation procedure correctly identifies when a policy or search is performs better, resulting in the strong portfolio planner, \portfolio{}, with a total coverage of 630. 
As to be discussed in the next question, \portfolio{} takes advantage of the \policyHeuristicMode{} setting which improves upon the domains where \plmp{} struggle, and also upon \hlmp{} for search.

\ssection{(2) How important is soundness and completeness for planning performance?}
Recall that search with value functions is both sound and complete, while execution of our LM-generated policies is sound but not complete.
We note that soundness is an important property: previous works show that LM prompting of plans, which is neither sound nor complete, achieve low coverage on planning problems which are trivial to solve for symbolic planning systems~\cite{valmeekam.etal.2023,valmeekam.etal.2024}.
In this work, all approaches are sound so we study the impact of complete algorithms.
We recall from the previous question that \plmp{} and \portfolio{} are incomplete but achieve better overall performance over complete approaches.
However, as also noted previously, the performance of policies on some domains are very poor when the LM did not understand how to solve these tasks.

However, we observe that combining the policies into (complete) search with value functions as done in \policyHeuristicMode{} always matches or improves upon pure search with \heuristicMode{}.
Indeed, adding policy generated states into a queue strictly improves search in 8 out of 10 domains, including domains on which the policy \policyMode{} by itself struggles, and improves the overall coverage of \hlmp{} from 397 to 424.
Furthermore, the validation procedure correctly identifies when to use search or policy execution which suggests that combining the best of complete and incomplete algorithms while maintaining soundness is a promising approach.
We refer to \Cref{app:corr} for a statistical analysis of a positive correlation between validation and test performance that supports this validation procedure.

\ssection{(3) Are LMs planning over word semantics or logical symbols?}
In order to address\nbsm{I'm not sure the ablation allows us to answer the question definitively.} this question, we perform the ablation introduced by~\citet{silver.etal.2024} by replacing all type, predicate, function, schema and object names in all PDDL files with nondescriptive symbols; e.g. predicates are renamed to \pddl{p1}, \pddl{p2}, \dots, actions to \pddl{a1}, \pddl{a2}, \ldots, objects to \pddl{o1}, \pddl{o2}, \ldots, etc.
\Cref{fig:symbolic} illustrates the coverage results of LM-generated value functions for search (\hlmp{}) and policies for rollout (\plmp{}) on PDDL input files with and without semantic names.

\newcommand{\sem}{\mathrm{sem}}
\newcommand{\sym}{\mathrm{sym}}
\begin{wrapfigure}{r}{0.5\textwidth}

    \vskip-2.5ex

    \covTableSize
    \tabcolsep 3pt
    \newcommand{\m}[1]{\cellcolor{green!40}{{#1}}}
    \newcommand{\f}[1]{\cellcolor{red!40}{{#1}}}
    \makeatletter\def\@captype{table}\makeatother
    \caption{Coverage ($\uparrow$) with ($\sem$) and without ($\sym$) semantic names. Green/red cells indicate where $\sym$ solves at least 3 problems more/fewer problems than its $\sem$ counterpart for a domain.}
    \label{fig:symbolic}
    \begin{tabularx}{0.5\textwidth}{l *{11}{Y}}
        \toprule
                   & \bl   & \ch    & \fe    & \fl & \mi    & \ro    & \sa    & \so & \sp    & \tr    & $\Sigma$ \\
        \midrule
        $V_\sem$   & 
33 & 15 & 59 & 2 & 63 & 32 & 60 & 32 & 46 & 55 & 397 \\
        $V_\sym$   &
33 & \m{24} & 61 & 1 & \m{70} & 34 & \f{48} & 30 & \m{63} & \f{42} & 406 \\
        \midrule
        $\pi_\sem$ &
90 & 11 & 90 & 0 & 90 & 12 & 90 & 0 & 90 & 90 & 563 \\
        $\pi_\sym$ &
\f{1} & 12 & 90 & 0 & 90 & \m{46} & \f{6} & 0 & 90 & 89 & 424 \\
        \bottomrule
    \end{tabularx}

    \vskip-0.5ex

\end{wrapfigure}

Surprisingly and contrary to previous works performing similar experiments~\cite{valmeekam.etal.2023a,silver.etal.2024}, we observe that in multiple settings, LM-generated programs perform better without than with semantic names in the PDDL inputs.
LM-generated value functions perform better by at least 3 problems without semantic names on 3 domains and worse on 2 (see green and red cells in \Cref{fig:symbolic}).
The case for policies is performing better on 1 domain and worse on 2 domains.
Specifically, removing semantic names often improves performance for value function generation but decreases for policy generation.
However, there are still 3 domains for which policies can solve all problems even without semantic names in PDDL inputs (\fe, \mi, \sp).
It is unclear exactly why this may be the case but the results mirror related work studying whether LMs can reason without comprehensible natural language inputs.
Indeed, \citet{pfau.etal.2024} showed that removing word semantics from reasoning tokens in LMs, instead of the input reasoning problem as in our work, has minor impact to reasoning performance, while \citet{stechly.etal.2025} showed that intermediate reasoning tokens may not reflect human-like or algorithm-interpretable trace semantics.  \nbsm{It would be fun to repeat the experiment using "gensym" to generate random symbols rather than using p1, p2, p3, a1, a2, a3, etc. because "p" does mean predicate and "a" does mean action, and "o" does mean object, and we are synthesizing code, so it would be better for each predicate, symbol, and object name to not identify it as being of a particular type and also numbering them so you know it's something one might iterate over. I can imagine that this would help code synthesis, even more than disparate word names.  BOTH experiments will be interesting for different reasons to futher investigate this phenomenon.}
\nbdc{:thumbs-up:}


\subsection*{Limitations}
Although our proposed approaches, specifically \plmp{} and \portfolio{} in \Cref{tab:coverage}, achieve the highest total coverage, they are not necessarily the best performing planners across all domains and metrics.
This fact may change with the improvement of LMs over time as they understand how to solve more complex planning problems which are still tractable to solve satisficingly, such as the Childsnack, Floortile and Rovers domains.

Another limitation of the synthesised policies is that they have no completeness or termination guarantee.
Although it is possible to guarantee such properties by combining them with search or by using them as an epsilon-greedy policy in the case of domains with reversible actions, the latter approach comes at a cost of extremely poor plan quality.

Indeed, we have yet to discuss plan quality.
\Cref{fig:plan_quality} compares the solution qualities of various approaches.
We direct the reader to the 2nd plot (\plmp{} compared against \lama{}) and note that the LM-generated policy returns inferior plans on the Blocksworld and Transport domains, sometimes up to 100 times worse for several Transport problems.
When analysing the plans and generated code, the policy sometimes randomly selects unnecessary actions that undo previous actions.
On the other hand, LM-generated value functions perform similarly to PDDL planners in terms of plan quality (1st plot: \hlmp{} compared against \lama{}).
Similarly, the effect of input representation on plan quality of LM-generated value functions is neglible (3rd plot: $V_{\sem}$ compared against $V_{\sym}$).
However, there is more variance for the case of policies, especially Rovers where although $\pi_{\sym}$ solves more problems, it achieves significantly worse plans (4th plot: $\pi_{\sem}$ compared against $\pi_{\sym}$).

Lastly, although replacing semantic names with arbitrary symbols\nbsm{I'm not sure the term "semantic names" was used consistently from the outset. Perhaps consider checking (or not :-))} in our PDDL encodings does not degrade the performance of value function generation, doing so significantly decreases the performance of generated policies on 2 domains (Blocksworld and Satellite).
Furthermore, it is not obvious why the performance of LM-generated value functions improves when semantic names are removed which warrants further investigation. \nbsm{One question for us (for the presentation) is whether the change in performance is statistically significant. }

\begin{figure}
  \centering
  \includegraphics[width=\textwidth]{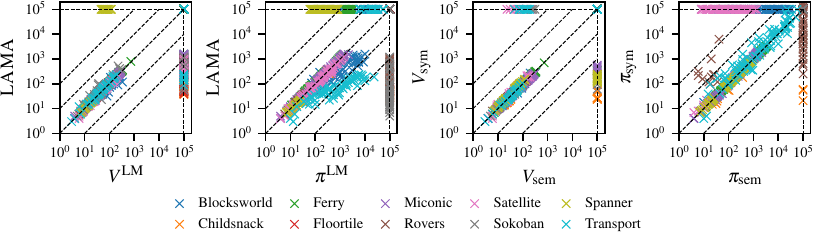}
  \caption{
    Returned plan cost ($\downarrow$) of planners labelled in the $x$ and $y$-axes in log scale.
    Problems that were not solved by one planner has their respective metric set to the axis limit.
    Points on the top left triangle favour the $x$-axis planner while points on the bottom right triangle favour the $y$-axis planner.
  }
  \label{fig:plan_quality}
\end{figure}

\section{Related Work}\label{sec:literature}
\ssection{LMs for PDDL Planning}
Previous works have shown that querying LMs directly to output plans~\cite{valmeekam.etal.2023,valmeekam.etal.2024} or using LMs themselves as value functions~\cite{katz.etal.2024} result in poor planning performance and LM efficiency.
Instead, LMs have shown more success in hybrid systems~\cite{kambhampati.etal.2024} such as those which generate or leverage PDDL models from natural language to be solved by or with the aid of PDDL planning technology~\cite{collins.etal.2022,lin.etal.2023,xie.etal.2023,liu.etal.2023,guan.etal.2023,oswald.etal.2024,liu.etal.2025}.
The closest related works to ours that study generalised planning via LM-generated programs~\cite{silver.etal.2024} and via LM-generated value functions~\cite{tuisov.etal.2025,correa.etal.2025} provide significantly stronger planning performance than standalone LM planners.
In the former work by~\citet{silver.etal.2024}, LMs are queried to generate code which aims to directly synthesise a plan for a PDDL planning problem of a given domain.
Although external verifiers are used to improve the code in a validation stage, the approach is not sound for test problems.
The latter works by~\citet{tuisov.etal.2025} and~\citet{correa.etal.2025} query LMs to generate a value function for use in heuristic search by leveraging the PDDL input structure and existing planners.
\citet{tuisov.etal.2025} generate a value function for each new PDDL problem, whereas we and \citet{correa.etal.2025} only use one value function for all PDDL problems within the same PDDL domain.
Our work differs from these two works as we also synthesise \emph{sound policies} via LMs as opposed to value functions for search or unsound programs.
We further perform an ablation study that shows the LMs used in our framework do not suffer from performance when replacing semantic words in PDDL problem inputs with meaningless symbols.

\ssection{Learning Reusable Value Functions and Policies}
AI planning researchers have been studying representations and approaches for learning reusable value functions and policies for PDDL planning which can generalise to unseen problems of arbitrary numbers of objects~\cite{celorrio.etal.2012,illanes.mcilraith.2019,celorrio.etal.2019}\nbsm{cite Loom?}.
This is analogous to reinforcement learning whose approaches can be categorised into value function learning~\cite{watkins.dayan.1992,mnih.etal.2013,mnih.etal.2015,hasselt.etal.2016}, policy gradient methods~\cite{williams.1992,schulman.etal.2015,schulman.etal.2017}, and actor-critic methods which combine value function and policy learning~\cite{konda.tsitsiklis.1999,sutton.etal.1999a,mnih.etal.2016,lillicrap.etal.2016}.
Recent works in machine learning for PDDL planning involve learning value functions~\cite{shen.etal.2020,karia.srivastava.2021,gehring.etal.2022,chen.etal.2024b,chen.etal.2024a,hao.etal.2024,chen.thiebaux.2024a} and policies~\cite{toyer.etal.2018,toyer.etal.2020,staahlberg.etal.2022,wang.thiebaux.2024}.
Our work is one of the few in this body of work which learns to evaluate both states and actions~\cite{wang.trevizan.2025}.

\ssection{Generalised Planning}
Generalised planning (GP) aims to compute programs that can solve families of related planning problems.
GP stemmed from synthesising programs containing conditionals and loops~\cite{levesque.2005,srivastava.etal.2008,segoviaaguas.etal.2024} from which researchers found various approaches for representing programs such as with memoryless finite-state controllers~\cite{bonet.etal.2009,bonet.etal.2010,hu.degiacomo.2011,aguas.etal.2018} or policies derived from lifted rules~\cite{khardon.1999,martin.geffner.2004,srivastava.etal.2011,illanes.mcilraith.2019,frances.etal.2019,yang.etal.2022a,drexler.etal.2022,hofmann.geffner.2024}.
Key attributes of symbolic GP approaches include guarantees of soundness, completeness, and termination of policies, some or all of which are usually not guaranteed by LM or deep learning architectures.
Our work studies the effect of soundness and completeness for LMs for GP.


\section{Conclusion}
\label{sec:conclusion}
We introduce a language model (LM) planner, \lmplan{}, that generates Python programs as value functions and sound policies for PDDL planning.
Conducted experiments show that \lmplan{} achieves strong planning performance relative to state-of-the-art planners and recent LM approaches.
We also identify that, surprisingly, \lmplan{} sometimes show better planning performance over purely symbolic representations of planning problems.
This observation challenges previous hypotheses that LMs cannot reason over meaningless symbols and is worth much further exploration.

\subsection*{Acknowledgements}
This work was carried out when all three authors were research interns at the Vector Institute for AI, Toronto, Canada.
We gratefully acknowledge funding from the Natural Sciences and Engineering Research Council of Canada (NSERC) and the Canada CIFAR AI Chairs Program.
Resources used in preparing this research were provided, in part, by the Province of Ontario, the Government of Canada through CIFAR, and companies sponsoring the Vector Institute.
We thank Felix Dangel for feedback on the figures in the paper.
Finally, the first three authors acknowledge the Nando's team on Bay Street for the copious amounts of Extra Hot chicken that helped fuel this work.

\bibliography{llmplan.bib}
\bibliographystyle{alpha}

\clearpage
\appendix

\section{Money Usage Estimate} \label{app:money}
\begin{table}[h!]
  \caption{
    Money usage estimate from calling LM APIs for all experiments;
    400 API calls were made per model: 200 for PDDL files with semantic names, and 200 for without semantic names.
  }
  \label{tab:money}
  \begin{tabularx}{\columnwidth}{l Y Y}
    \toprule
    Model & Average \$USD per query & Total \$USD\\
    \midrule
    DeepSeek-R1 & 0.00655 & 2.62 \\
    Gemini 2.0 Flash & 0 & 0 \\
    Gemini 2.5 Flash Preview 04-17 & 0 & 0 \\
    \bottomrule
  \end{tabularx}
\end{table}

\newcommand{\dsr}{DS-R1}
\newcommand{\geminia}{Gem2.0}
\newcommand{\geminib}{Gem2.5}

\newcommand{\dsrVal}[1]{\cellcolor{red!30}{#1}}
\newcommand{\geminiaVal}[1]{\cellcolor{blue!30}{#1}}
\newcommand{\geminibVal}[1]{\cellcolor{green!30}{#1}}

\newcommand{\dsrCell}{\dsrVal{\dsr}}
\newcommand{\geminiaCell}{\geminiaVal{\geminia}}
\newcommand{\geminibCell}{\geminibVal{\geminib}}

\section{Best LM Generated Programs Chosen by Validation} \label{app:best-programs}
\begin{table}[h!]
  \caption{
    Best LM-generated program chosen by the validation procedure described in \Cref{sec:experiments} for experiments of PDDL files with (semantic) and without (symbolic) semantic names.
    The LM models used in the experiments are DeepSeek-R1 ({\color{red}\dsr{}}), Gemini 2.0 Flash ({\color{blue}\geminia{}}), and Gemini 2.5 Flash Preview 04-17 ({\color{green}\geminib{}}).
  }
  \label{tab:validation}
  \newcommand{\limer}[1]{\multicolumn{1}{c}{#1}}
  \begin{tabularx}{\columnwidth}{l X X X X}
    \midrule
    & \multicolumn{2}{c}{Value Functions} & \multicolumn{2}{c}{Policies} \\
    Domain & \limer{semantic} & \limer{symbolic} & \limer{semantic} & \limer{symbolic} \\
    \midrule
    Blocksworld & \geminibCell{} & \geminibCell{} & \geminibCell{} & \geminiaCell{} \\
    Childsnack & \dsrCell{} & \geminibCell{} & \dsrCell{} & \dsrCell{} \\
    Ferry & \dsrCell{} & \dsrCell{} & \geminibCell{} & \geminibCell{} \\
    Floortile & \geminibCell{} & \geminibCell{} & \dsrCell{} & \dsrCell{} \\
    Miconic & \geminiaCell{} & \dsrCell{} & \geminiaCell{} & \geminibCell{} \\
    Rovers & \geminibCell{} & \dsrCell{} & \geminibCell{} & \geminiaCell{} \\
    Satellite & \geminibCell{} & \geminibCell{} & \geminibCell{} & \geminibCell{} \\
    Sokoban & \geminibCell{} & \geminibCell{} & \dsrCell{} & \geminibCell{} \\
    Spanner & \geminiaCell{} & \dsrCell{} & \geminiaCell{} & \geminibCell{} \\
    Transport & \dsrCell{} & \geminibCell{} & \geminiaCell{} & \dsrCell{} \\
    \bottomrule
  \end{tabularx}
\end{table}

\section{LM Program Generation Times}\label{sec:generation_time}
\begin{table}[h!]
  \caption{
    Time taken in seconds for LMs in \Cref{tab:validation} to generate a program.
  }
  \label{tab:times}
  \newcommand{\limer}[1]{\multicolumn{1}{c}{#1}}
  \begin{tabularx}{\columnwidth}{l Y Y Y Y}
    \midrule
    & \multicolumn{2}{c}{Value Functions} & \multicolumn{2}{c}{Policies} \\
    Domain & \limer{semantic} & \limer{symbolic} & \limer{semantic} & \limer{symbolic} \\
    \midrule
    Blocksworld & \geminibVal{78.0} & \geminibVal{108.2} & \geminibVal{107.4} & \geminiaVal{3.4} \\
    Childsnack & \dsrVal{329.8} & \geminibVal{119.7} & \dsrVal{260.2} & \dsrVal{277.3} \\
    Ferry & \dsrVal{310.6} & \dsrVal{419.2} & \geminibVal{42.6} & \geminibVal{168.5} \\
    Floortile & \geminibVal{118.4} & \geminibVal{104.5} & \dsrVal{528.2} & \dsrVal{474.2} \\
    Miconic & \geminiaVal{1.5} & \dsrVal{375.6} & \geminiaVal{4.8} & \geminibVal{98.2} \\
    Rovers & \geminibVal{50.8} & \dsrVal{331.0} & \geminibVal{98.6} & \geminiaVal{3.0} \\
    Satellite & \geminibVal{93.7} & \geminibVal{102.1} & \geminibVal{109.8} & \geminibVal{151.0} \\
    Sokoban & \geminibVal{65.7} & \geminibVal{152.3} & \dsrVal{527.7} & \geminibVal{33.0} \\
    Spanner & \geminiaVal{2.9} & \dsrVal{477.3} & \geminiaVal{2.2} & \geminibVal{60.3} \\
    Transport & \dsrVal{434.2} & \geminibVal{113.7} & \geminiaVal{3.6} & \dsrVal{394.7} \\
    \bottomrule
  \end{tabularx}
\end{table}

\clearpage
\section{Correlation Between Validation and Test Performance}\label{app:corr}
We perform a statistical correlation analysis between the \ref{ssec:valmetric} on validation problems and coverage on testing problems. Specifically, for each benchmark domain and LM model, we have multiple \heuristicMode{}/\policyMode{} programs which we run heuristic search/policy execution on 11 small planning problems (every 9th problem in the training split) and 10 testing problems (every 9th problem in the testing split). These experiments were performed after the aforementioned experiments to prevent overfitting to the test set. \Cref{fig:correlation} illustrates the results.

Interestingly, for both value functions and policies, there is a statistically significant ($p < 0.05$) positive Pearson correlation ($\rho \simeq 1$) between the validation metric and test performance, which suggests that validation sets provide a useful proxy for estimating testing performance. However, the correlation is not perfect as we noted that the programs with the best validation scores did not provide the best overall coverage. The correlation remains high when conditioning on individual domains for where enough unique samples were provided, with an exception being Childsnack policies. We lastly note that performing the validation procedure for the \portfolio{} configuration correctly identifies whether value functions or policies perform better on each domain, a priori to running the test-time experiments.


\begin{figure}[h!]
  \centering
  \caption{
    \emph{Left}: correlation coefficients conditioned on domain.
    \emph{Right}: average coverage ($y$-axis) vs. \ref{ssec:valmetric} ($x$-axis) for LM-generated programs.
  }
  \setlength{\tabcolsep}{1pt}
  \raisebox{2.2\height}{
    \begin{threeparttable}
      \scriptsize
      \begin{tabularx}{0.5\textwidth}{@{} l *{10}{Y} @{}}
        \toprule
        & \bl & \ch & \fe & \fl & \mi & \ro & \sa & \so & \sp & \tr \\
        \midrule
        \heuristicMode{}
        & 1.0 & 0.8 & 1.0 & \zerocell{} & 1.0 & 1.0 & 1.0 & 1.0 & 1.0 & 1.0 \\
        \policyMode{}
        & 1.0 & 0.4 & 1.0 & \zerocell{} & 1.0 & 0.8 & 0.9 & \zerocell{} & 1.0 & 1.0 \\
        \bottomrule
      \end{tabularx}%
      \begin{tablenotes}
      \item Cells with \zerocell{} values indicate that there were fewer than 2 unique datapoints for each domain.
      \end{tablenotes}
    \end{threeparttable}
  }
  \includegraphics[width=0.45\textwidth]{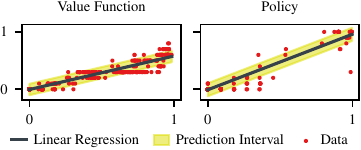}
  \label{fig:correlation}
\end{figure}

\end{document}